%% file: iclr2023_conference.tex
\documentclass{article} 
\usepackage{arxiv}

\input{math_commands.tex}

\usepackage{graphicx}
\usepackage{hyperref}
\usepackage{url}
\usepackage{amssymb}  

\usepackage{amsthm}
\usepackage{color}
\usepackage{subcaption}
\newtheorem{theorem}{Theorem}[section]
\newtheorem{lemma}{Lemma}[section]
\newtheorem{assumption}{Assumption}[section]
\usepackage{algorithm}
\usepackage{algpseudocode}

\title{One-Shot Averaging for Distributed TD($\lambda$) Under  Markov Sampling}

\author{Haoxing Tian, Ioannis Ch. Paschalidis, Alex Olshevsky\\
Department of Electrical and Computer Engineering\\
Boston University\\
Boston, MA 02215, USA \\
\texttt{\{tianhx, yannisp, alexols\}@bu.edu} \\
}

\begin{document}

\maketitle

\begin{abstract}
We consider a distributed setup for reinforcement learning, where each agent has a copy of the same Markov Decision Process but transitions are sampled from the corresponding Markov chain independently by each agent.   We show that in this setting, we can achieve a linear speedup for TD($\lambda$), a family of popular methods for policy evaluation, in the sense that $N$ agents can evaluate a policy $N$ times faster provided the target accuracy is small enough. Notably, this speedup is achieved by ``one shot averaging,'' a procedure where the agents run TD($\lambda$) with Markov sampling independently and only average their results after the final step. This significantly reduces the amount of communication required to achieve a linear speedup relative to previous work.
\end{abstract}

\section{Introduction}

Actor-critic method achieves state-of-the-art performance in many domains including robotics, game playing, and control systems (\cite{lecun2015deep,mnih2016asynchronous,silver2017mastering}). Temporal Difference (TD) Learning may be thought of as a component of actor critic, and better bounds for TD Learning are usually ingredients of actor-critic analyses. We consider the problem of policy evaluation in reinforcement learning: given a Markov Decision Process (MDP) and a policy, we need to estimate the value of each state (expected discounted sum of all future rewards)   under this policy. Policy evaluation is important because it is effectively a subroutine of many other algorithms such as policy iteration and actor-critic. The main challenges for policy evaluation are that we usually do not know the underlying MDP directly and can only interact with it,  and that the number of states is typically too large forcing us to maintain a low-dimensional approximation of the true vector of state values.  

We focus here on the simplest class of methods overcoming this set of challenges, namely TD methods with linear function approximation. These methods try to maintain a low dimensional parameter which is continuously updated based on observed rewards and transitions to maintain consistency of estimates across states. The proof of convergence for these methods was first given in  \cite{tsitsiklis1996analysis-nips}. 

In this paper, we focus on the multi-agent version of policy evaluation: we consider $N$ agents that have a copy of the same MDP and the same policy, but transitions in the MDP by different agents are independent. The question we wish to ask is whether the agents can cooperate to evaluate the underlying policy $N$ times faster, since now $N$ transitions are generated per unit time. 

Although there is some  previous work on distributed temporal difference methods (e.g., \cite{doan2019finite, sun2020finite, wang2020decentralized}), this question has only been considered in the recent papers \cite{khodadadian2022federated, wang2023federated, zhang2024finite, liu2023distributed}. The answer was positive in both \cite{khodadadian2022federated, wang2023federated} in a ``federated learning'' setting, provided the nodes have $N$ rounds of communication with a central server before time $T$, with environment heterogeneity additionally considered in \cite{wang2023federated}. In \cite{zhang2024finite}, the answer was also positive (i.e., linear speedup was obtained) under a distributed erasure model where each node communicated with neighbors in a graph a constant fraction of time, leading to $O(T)$ communications in $T$ steps. Our previous work \cite{liu2023distributed} established that, in fact, only one communication round with a central server is sufficient in the case of i.i.d. observations and TD(0), the most basic method within the temporal difference family.  This was accomplished via the ``one-shot averaging'' methods where the $N$ agents just ignore each other for $T$ steps, and then simply average their results. Further, the final averaging step could be replaced with $O(\log T)$ rounds of an average consensus method.  

The i.i.d. observation assumption is a limiting feature of our previous work in \cite{liu2023distributed}: it is assumed that at each time, we can generate a random state from the underlying MDP. This is convenient for analysis but rarely satisfied in practice. 

In this paper, our contribution is to show that one-shot averaging suffices to give a linear speedup  {\em without} the i.i.d. assumption and for the more general class of temporal difference methods TD($\lambda$) (precise definitions are given later). Our method of proof is new and does not overlap with the arguments given in our previous work. 

\section{Background}

\subsection{Markov Decision Process (MDP)}

A finite discounted-reward MDP can be described by a tuple $(S, A, P_{\rm env}, r, \gamma)$, where $S$ is the state-space whose elements are vectors, with $s_0$ being the starting state;  $A$ is the action space; $P_{\rm env} = (P_{\rm env}(s'|s, a))_{s,s'\in S, a \in A}$ is the transition probability matrix, where $P_{\rm env}(s'|s, a)$ is the probability of transitioning from $s$ to $s'$ after taking action $a$;  $r:S \times S \to \mathbb{R}$ is the reward function, where $r(s, s')$ associates a deterministic reward with each state transition; and $\gamma \in (0,1)$ is the discount factor.

A policy $\pi$ is a mapping $\pi : S \times A \to [0,1]$ where $\pi(a|s)$ is the probability that action $a$ is taken in state $s$. Given a policy $\pi$, the state transition matrix $P_{\pi} = \left( P_{\pi}(s'|s) \right)_{s, s' \in S}$ and the state reward function $r_{\pi}(s)$ is defined as 
\begin{small}
\begin{align*}
    P_{\pi}(s'| s) = \sum_{a \in A} P_{\rm env}(s'|s, a) \pi(a|s),\quad r_{\pi}(s) = \sum_{s' \in S} P_{\pi}(s'| s) r(s, s').
\end{align*}
\end{small}Since the policy is fixed throughout the paper, we will omit the subscript $\pi$ and thus write $P(s'|s)$ and $r(s)$ instead of $P_{\pi}(s'| s)$ and $r_{\pi}(s)$. 

The stationary distribution $\mu$ is a nonnegative vector with coordinates summing to one and satisfying $\mu^T = \mu^T P$. The Perron-Frobenius theorem guarantees that such a stationary distribution exists and is unique subject to some conditions on $P$, e.g., aperiodicity and irreducibility \cite{gantmacher1964theory}. The entries of $\mu$ are denoted by $\mu(s)$. We also define $D = \text{diag}(\mu(s))$ as the diagonal matrix whose elements in the main diagonal are given by the entries of the stationary distribution $\mu$.

The value function $V_{\pi}^*(s)$ is defined as 
$$
V_{\pi}^*(s) = \mathbb{E}_{\pi} \left[\sum_{t=0}^{+\infty} \gamma^{t} r(s_t) \right]
$$
where $\mathbb{E}_{\pi}$ stands for the expectation when actions are chosen according to policy $\pi$. Since the MDPs is finite, it is without loss of generality to assume a bound on rewards.
\begin{assumption} \label{a: reward bound}
For any $s, s' \in S \times S$, 
$$
|r(s,s')| \le r_{\rm max}.
$$ 
\end{assumption}

\subsection{Value Function Approximation}

Given a fixed policy, the problem is to efficiently estimate $V^*_{\pi}$. We consider a linear function approximation architecture $V_{\pi, \theta}(s) = \phi(s)^T \theta$, where $\phi(s) \in \mathbb{R}^d$ is a feature vector for state $s$ and $\theta \in \mathbb{R}^d$. Without loss of generality, we assume $||\phi(s)|| \leq 1$ for all states $s$. For simplicity, we define $V_{\pi, \theta} = (V_{\pi, \theta}(s))_{s \in S}$ to be a column vector, $\Phi = \left( \phi(s)^T \right)_{s \in S}$ to be a $|S| \times d$ matrix and $R = \left( r(s) \right)_{s \in S}$ to be a column vector. We are thus trying to approximate $V_{\pi, \theta} \approx \Phi \theta$. We make the following assumptions:

\begin{assumption}[Input assumption] \label{a: input assumption}
It is standard to assume the following statements hold:
\begin{itemize}
    \item The matrix $\Phi$ has linearly independent columns. 
    \item The stationary distribution $\mu(s) > 0, \forall s \in S$. 
\end{itemize}
\end{assumption}

\subsection{Distributed Model and Algorithm}

We assume that there are $N$ agents and each agent shares the same tuple $(S, A, P_{\rm env}, r, \gamma)$ as well as the same fixed policy $\pi$. However, each agent independently samples its trajectories and updates its own version of a parameter $\theta_t$. 

We will study an algorithm which mixes TD learning and one-shot averaging: after all agents finish $T$ steps, they share their information and compute the average parameter as the final result. These agents do not communicate before the final step. The averaging can take place using average consensus (using any average consensus algorithm) or, in a federated manner, using a single communication with a coordinator.  

We next spell out the details of our algorithm. Every agent runs TD(0) with Markov sampling as follows. Agent $i$ generates an initial state $s_0(i)$ from some initial distribution. It also maintains an iterate $\theta_t(i)$, initialized arbitrarily. At time $t$, agent $i$ generates a transition according to $P$. It then computes the so-called TD-error
$$\delta_{s,s'}(\theta_t(i)) = r(s,s') + \gamma \phi(s')^T \theta_t(i) - \phi(s)^T \theta_t(i),
$$
with $s=s_t(i), s'=s_{t+1}(i)$ coming from the transition it just generated; and then updates
\begin{equation}
    \label{eq: td 0 update}
   \theta_{t+1}(i) = \theta_t(i) + \alpha_t g_{s_t(i), s_{t+1}(i)} (\theta_t(i))
\end{equation}
where the update direction is 
$$
g_{s,s'} (\theta) = \delta_{s,s'} (\theta) \phi(s).
$$
At the end, the agents average their results. The complete algorithm is shown in Algorithm \ref{alg: td 0}.

\begin{algorithm}
\begin{small}
\caption{TD(0) with local state} \label{alg: td 0}
\begin{algorithmic}
\Require Iterations $T$, learning rate $\alpha_t$.
\State Initialize $\theta_{0}(i)$ for every agent $i \in \{1,2,\ldots,N\}$. 
\For {$t \in \{0,1,\ldots,T-1\}$}
    \For {$i \in \{1,2,\ldots,N\}$}
        \State Compute $\delta_{s_t(i), s_{t+1}(i)}(\theta_t(i))$ and $g_{s_t(i), s_{t+1}(i)}(\theta_t(i))$. 
        \State Execute $$\theta_{t+1}(i) = \theta_t(i) + \alpha_t g_{s_t(i), s_{t+1}(i)} (\theta_t(i)).$$
    \EndFor
\EndFor
\State Return the average $$\bar{\theta}_T = \frac{1}{N} \sum_{i=1}^{N} \theta_T(i).$$
\end{algorithmic}
\end{small}
\end{algorithm}

We will further use $\bar{g}$ to denote the expectation of $g_{s,s'}$ assuming $s$ are sampled from the stationary distribution $\mu$ and $s'$ is generated by taking a step from $P$. We use $\mathbb{E}_{\rm I}$ to denote this expectation. Therefore, 
$$
\bar{g}(\theta) = \mathbb{E}_{\rm I} \left[ g_{s,s'}(\theta) \right].
$$
We can also rewrite $\bar{g}$ in matrix notation: 
$$
\bar{g}(\theta) = \Phi^T D (R+(\gamma P-I)\Phi \theta).
$$

In order to perform our analysis, we need to define the stationary point. We adopt the classic way of defining such point as shown in \cite{tsitsiklis1996analysis-nips}. We call $\theta^*$ the stationary point if $\bar{g}(\theta^*)=0$. According to \cite{bertsekas1996neuro}, in matrix notation, it is equivalent to say that $\theta^*$ satisfies the following:
\begin{equation}
\label{eq: stationary point for td 0}
    \Phi^T D \left( R + (\gamma P - I) \Phi \theta^* \right) = 0.
\end{equation}

Naturally, each agent can easily compute $\theta^*$ by running TD(0) for infinite times and simply ignore all the other agents. However, this ignores the possibility that agents can benefit from communicating with each other.

We next focus on TD($\lambda$), which is a popular generalization of the conceptually simpler TD(0) and attains better performance with an appropriate choice of $\lambda$ \cite{sutton2018reinforcement}. For any fixed $\lambda \in [0,1]$, TD($\lambda$) executes the update
\begin{equation}
    \label{eq: td lambda update}
    \theta_{t+1}(i) = \theta_{t}(i) + \alpha_t x_{s_{t}(i), s_{t+1}(i)}(\theta_{t}(i), z_{0:t}).
\end{equation}
Here, $z_{0:t}$ is called eligibility trace and $x_{s_{t}(i), s_{t+1}(i)}(\theta_{t}(i), z_{0:t})$ is given by
$$
z_{0:t} = \sum_{k=0}^t (\gamma \lambda)^k \phi(s_{t-k}(i)), \quad x_{s_{t}(i), s_{t+1}(i)}(\theta_{t}(i), z_{0:t}) = \delta_{s_{t}(i), s_{t+1}(i)}(\theta_{t}(i)) z_{0:t}
$$
The complete algorithm is shown in Algorithm \ref{alg: td lambda}.

\begin{algorithm}
\begin{small}
\caption{TD($\lambda$) with local state} \label{alg: td lambda}
\begin{algorithmic}
\Require Iterations $T$, learning rate $\alpha_t$.
\State Initialize for every agent by $\theta_0(i), i \in \{1,2,\ldots,N\}$. 
\For {$t \in \{0,1,\ldots,T-1\}$}
    \For {$i \in \{1,2,\ldots,N\}$}
        \State Compute $\delta_{s_{t}(i), s_{t+1}(i)}(\theta_{t}(i))$ and $x_{s_{t}(i), s_{t+1}(i)}(\theta_{t}(i), z_{0:t})$. 
        \State Execute  $$\theta_{t+1}(i) = \theta_{t}(i) + \alpha_t x_t(\theta_{t}(i), z_{0:t}).$$
    \EndFor
\EndFor
\State Return the average $$\bar{\theta}_T = \frac{1}{N} \sum_{i=1}^{N} \theta_T(i).$$
\end{algorithmic}
\end{small}
\end{algorithm}
For convenience of the analysis, we define the eligibility trace going back to minus infinity: $$
z_{-\infty:t} = \lim_{t \to \infty} z_{0:t} = \sum_{k=0}^\infty (\gamma \lambda)^k \phi(s_{t-k}(i)).
$$
We also introduce the operator $T_{\pi}^{(\lambda)}$ which is defined as 
\begin{align*}
    \left( T_{\pi}^{(\lambda)} V \right)(s) 
    = (1-\lambda) \sum_{k=0}^\infty \lambda^k \cdot \mathbb{E} \left[ \sum_{t=0}^k \gamma^t r(s_{t}(i),s_{t+1}(i)) + \gamma^{k+1} V(s_{k+1}(i)) \mid s_0(i) = s\right].
\end{align*}
In matrix notation:
\begin{equation}
\label{eq: t lambda}
T_{\pi}^{(\lambda)} V = (1-\lambda) \sum_{k=0}^\infty \lambda^k \left( \sum_{t=0}^k \gamma^t P^t R + \gamma^{k+1} P^{k+1} V \right).
\end{equation}
We denote $\bar{x}(\theta_{t}(i)) = \mathbb{E}_{\rm I} \left[ x_{s, s'}(\theta_{t}(i), z_{-\infty:t}) \right]$. This expectation assumes that $s_i(k) \sim \mu,\forall k$, and that the history of the process then extends to $-\infty$ according to a distribution consistent with each forward step being taken by $P$; for more details, see \cite{tsitsiklis1996analysis-nips}. We also note that Lemma 8 in \cite{tsitsiklis1996analysis-nips} implies
$$
\bar{x}(\theta_t) = \Phi^T D \left( T_{\pi}^{(\lambda)}(\Phi \theta) - \Phi \theta \right).
$$

We call $\theta^*$ the stationary point if $\bar{x}(\theta^*) = 0$. In matrix notation, it is equivalent to say that $\theta^*$ satisfies the following:
\begin{equation}
    \label{eq: stationary point for td lambda}
    \Phi^T D \left( T_{\pi}^{(\lambda)}(\Phi \theta^*) - \Phi \theta^* \right) = 0.
\end{equation}
Finally, we define $\kappa = \gamma \frac{1-\lambda}{1-\gamma \lambda}$ which will be useful later. An obvious result is that $\kappa \le \gamma$. 

\subsection{Markov Sampling and Mixing}
\label{sec: markov sampling}

As mentioned in the previous section, the ideal way of generating $s_{t}(i)$ is to draw from stationary distribution $\mu$. However, the typical way is generating a trajectory is $s_{1}(i), \ldots, s_{T}(i)$. Every state in this trajectory is sampled by taking a transition $s_{t}(i) \sim P(\cdot | s_{t-1}(i))$ (and recall our policy is always fixed). This way of sampling is called Markov sampling, and we denote $\mathbb{E}_{\rm M}=\mathbb{E}_{s_{t}(i) \sim P(\cdot | s_{t-1}(i))}$. Analyzing algorithms under Markov sampling can be challenging since one cannot ignore the dependency on previous samples. The following ``uniform mixing'' assumption is standard \cite{bhandari2018finite}. We also note that this assumption always holds for irreducible and aperiodic Markov chains \cite{levin2017markov}.
\begin{assumption} \label{a: uniform mixing}
There are constants $m$ and $\rho$ such that
$$
||P(s_t \in \cdot|s_0) - \mu||_1 \le m \rho^t, \quad \forall t. 
$$
\end{assumption}
A key definition from the uniform mixing assumption is called the mixing time. We define the mixing time $\tau_{\rm min}(\epsilon)$:
$$
\tau_{\rm mix}(\epsilon) = \min \{ t \mid m \rho^t \le \epsilon \}.
$$
In this paper, we always set $\epsilon = \alpha_t$ which is the step-size at time $t$, typically $\alpha_t = \beta/(c+t)$, and simplify $\tau_{\rm mix}(\alpha_t)$ as $\tau_{\rm mix}$. An obvious result is that
$$
m \rho^t \le \alpha_t, \quad \forall t \ge \tau_{\rm mix}.
$$

\subsection{Convergence times for centralized TD(0) and TD($\lambda$)}

We now state the state-of-the-art results for the centralized case which are based on using ideas from gradient descent to analyze TD(0) and TD($\lambda$). These results are first proposed by \cite{bhandari2018finite} and they considered Projected TD Learning, where $\theta_t$ is projected onto a ball of fixed radius after the update is performed. Here, we are going to show that these result also holds when the projection step is removed and will use these results as a basis for comparison for our distributed results.

\begin{lemma} \label{l: result of linear td with markov}
In TD(0) with the Markov sampling, suppose Assumptions \ref{a: input assumption}, \ref{a: uniform mixing} hold and $t_{\rm th} = \max\{\tau_{\rm mix}, \frac{18}{(1-\gamma)^2 \omega^2}-1\}$. For a decaying stepsize sequence $\alpha_t = 2/(\omega (t+1) (1-\gamma))$, 
$$
\mathbb{E} \left[ ||\theta_{T+t_{\rm th}}(i) - \theta^*||^2 \right] \le \nu_{\rm central} \sim O \left( \frac{(\log T)^2}{T} \right), \quad \forall T \ge 0.
$$
\end{lemma}

We next discuss convergence times for TD($\lambda$). It is usually assumed that the algorithm extends back to negative infinity, and that every $s_{t}(i)$ has distribution $\mu$ (but the samples are, of course, correlated since each successive state is obtained by taking a step in the Markov chain $P$ from the previous one). Similarly as before, we define $\tau_{\rm mix}^{(\lambda)}(\epsilon)$ as
$$
\tau_{\rm mix}^{(\lambda)}(\epsilon) = \max \{ \tau_{\rm mix}(\epsilon), \tau'_{\rm mix}(\epsilon) \},
$$
where
$$
\tau_{\rm mix}(\epsilon) = \min \{ t \mid m \rho^t \le \epsilon \}, \quad \tau'_{\rm mix}(\epsilon) = \min \{ t \mid (\gamma \lambda)^t \le \epsilon \}.
$$
As before, we choose $\epsilon = \alpha_t$ and simplify $\tau_{\rm mix}^{(\lambda)}(\alpha_t)$ as $\tau_{\rm mix}^{(\lambda)}$. The same result applies to $\tau_{\rm mix}^{(\lambda)}$:
$$
\max \{ m \rho^t, (\gamma \lambda)^t \} \le \alpha_t, \quad \forall t \ge \tau^{(\lambda)}_{\rm mix}(\alpha_t).
$$
With these notations in place, we can now state the following result.

\begin{lemma} \label{l: result of td lambda with markov}
In TD($\lambda$) with the Markov sampling, suppose Assumptions \ref{a: input assumption}, \ref{a: uniform mixing} hold and $t_{\rm th}^{(\lambda)} = \max \{ \tau_{\rm mix}^{(\lambda)}, \frac{28}{(1-\kappa)^2\omega^2(1-\gamma \lambda)}-1\}$. For a decaying stepsize sequence $\alpha_t = 2/(\omega (t+1) (1-\kappa))$,
$$
\mathbb{E} \left[ \Vert \theta_{T+t_{\rm th}^{(\lambda)}}(i) - \theta^* \Vert^2 \right] \le \nu^{(\lambda)}_{\rm central} \sim O \left( \frac{(\log T)^2}{T} \right).
$$
\end{lemma}

\section{Main Result}

We now state our main result which claims a linear speed-up for both distributed TD(0) and distributed TD($\lambda$). Recall that we use $\theta_{t}(i)$ to denote the parameters of agent $i$ at time $t$, $\bar{\theta}_t = (\sum_i \theta_{t}(i))/N$ to denote the averaged parameters among all $N$ agents, and $\bar{\theta}_{t}(i) = \mathbb{E} \left[ \theta_{t}(i) \right]$ to denote the expectation of $\theta_{t}(i)$. We now have the following two theorems for TD(0) and TD($\lambda$) respectively. Notice that $\tilde{O}$ omits logarithm factors. 

\begin{theorem} \label{t: td with markov}
Suppose Assumptions \ref{a: input assumption} and \ref{a: uniform mixing} hold. Denote $t_0 = \max \{ \tau_{\rm mix}, \frac{8}{\omega \omega_{\rm I} (1-\gamma)} -1, t_{\rm th}\}$. With $\nu_{\rm central}$ in Lemma \ref{l: result of linear td with markov}, TD(0) with $\alpha_t =  2/(\omega (t+1) (1-\gamma))$ satisfies
\begin{align*}
    \mathbb{E} \left[ ||\bar{\theta}_{T+t_0} - \theta^* ||^2 \right] 
    \le \frac{1}{N} \nu_{\rm central} + \tilde{O} \left( \frac{1}{T^2} \right), \forall T \ge 0.
\end{align*}
\end{theorem}

\begin{theorem}
\label{t: td lambda with markov}    
Suppose Assumptions \ref{a: input assumption} and \ref{a: uniform mixing} hold. Denote $t_0^{(\lambda)} = \max \{ 2 \tau^{(\lambda)}_{\rm mix}, \frac{8}{\omega \omega_{\rm I}^{(\lambda)} (1-\kappa)} -1, t_{\rm th}^{(\lambda)}\}$. With $\nu^{(\lambda)}_{\rm central}$ in Lemma \ref{l: result of td lambda with markov}, TD($\lambda$) with $\alpha_t = 2/(\omega (t+1) (1-\kappa))$ satisfies
\begin{align*}
    \mathbb{E} \left[ ||\bar{\theta}_{T+t_0} - \theta^* ||^2 \right] 
    \le \frac{1}{N} \nu^{(\lambda)}_{\rm central} + \tilde{O} \left( \frac{1}{T^2} \right), \forall T \ge 0.
\end{align*}

\end{theorem}

In brief, the distributed version with $N$ nodes is $N$ times faster than the comparable centralized version for large enough $T$ (note that $v_{\rm central}$ and $v_{\rm central}^{(\lambda)}$ are $\tilde{O}(1/T)$ whereas the term that does not get divided by $N$ is $\tilde{O}(1/T^2)$ in both theorems). This significantly improves previous results from \cite{liu2023distributed}, which only showed this for TD(0). 

We note that the proofs given in this paper have no  overlaps with the proof from \cite{liu2023distributed}, which could not be extended to either TD($\lambda$) or Markov sampling (and here both extensions are done simultaneously). Instead, our analysis here is based on the following simple observation, at each step of which we just use independence plus elementary algebra:  
\begin{equation}
\label{eq: distributed analysis}
\begin{aligned}
    \mathbb{E} \left[ ||\bar{\theta}_T - \theta^* ||^2 \right] 
    \le & \frac{1}{N^2} \sum_{i=1}^N  \mathbb{E} \left[ \left\Vert \theta_{T}(i) - \theta^* \right\Vert^2 \right] + \frac{2}{N^2} \sum_{1 \le i < j \le N} \mathbb{E} \left[ \left(\theta_{T}(i) - \theta^* \right)^T \left(\theta_{T}(j) - \theta^* \right) \right] \\
    = & \frac{1}{N^2} \sum_{i=1}^N  \mathbb{E} \left[ \left\Vert \theta_{T}(i) - \theta^* \right\Vert^2 \right] + \frac{2}{N^2} \sum_{1 \le i < j \le N} \left( \bar{\theta}_{T}(i) - \theta^* \right)^T \left( \bar{\theta}_{T}(j) - \theta^* \right).
\end{aligned}
\end{equation}
Here, the last equality uses the fact that $\theta_{T}(i)$ are independent of each other since there is no communication during learning. This immediately implies a linear speed-up  if only we could prove that the first term dominates.  This is quite plausible, since the second term involves the convergence speed of the {\em expected} updates. 
In other words, all that is really needed is to prove that the expected update converges faster than the unexpected update. 

\section{Proof of Theorem \ref{t: td with markov}}

Before we go into the proof, we first introduce some notations. We rewrite $\bar{g}(\theta)$ in matrix notation as
\begin{equation}
\label{eq: g bar}
    \begin{aligned}
    \bar{g}(\theta) 
    = & \Phi^T D (I - \gamma P) \Phi \left( \theta^* - \theta \right)
    := \Sigma_{\rm I} \left( \theta^* - \theta \right). 
    \end{aligned}
\end{equation}
This matrix has some nice properties, which is pointed out in \cite{liu2021temporal,tian2022performance}. Indeed, we have the following lemmas whose proof we postpone: 

\begin{lemma}
\label{l: smallest eigenvalue of i.i.d. matrix}
There exists $\omega_{\rm I}>0$ such that 
$$
\inf_{\Vert x \Vert = 1} x^T \Sigma_{\rm I} x \ge \omega_{\rm I} \ge (1-\gamma) \omega.
$$
\end{lemma}

\begin{lemma}
\label{l: singular value of i.i.d. matrix}
For any $x$, $\Vert \Sigma_{\rm I} x \Vert^2 \le 4 \Vert x \Vert^2$.
\end{lemma}

In the rest of this section, we will omit $i$ since the analysis holds for all $i \in \{1, 2, \ldots, n\}$. For example, we will use $\theta_t$ instead of $\theta_t(i)$. 

First, we define 
$$
\bar{g}'(\theta_t) = \mathbb{E}_{\rm M} \left[ g_{s,s'}(\theta_t) \right]
$$
where $\mathbb{E}_{\rm M}$ is defined in Section \ref{sec: markov sampling}. We call $\bar{g}'(\theta_t) - \bar{g}(\theta_t)$ Markov noise and write it as
\begin{equation}
\label{eq: markov noise}
\begin{aligned}
    \bar{g}'(\theta_t) - \bar{g}(\theta_t) 
    = & \sum_{s_t, s_t'} \left( P_{t}(s_t|s_0) - \mu(s_t) \right) \phi(s_t) \cdot \left( r_t - \left( \phi(s_t) - \gamma P(s_t' | s_t) \phi(s_t') \right) \theta_t \right) \\
    = & \sum_{s_t, s_t'} \left( P_{t}(s_t|s_0) - \mu(s_t) \right) P(s_t' | s_t) g_{s_t, s_t'}(\theta_t),
\end{aligned}
\end{equation}
where $P_{t}(s'|s)$ stands for the $t$ step transition probability. To further address both the Markov noise and the recursion relations we will derive, we  need the following lemmas, whose proofs we also postpone.

\begin{lemma} \label{l: bounded optimal point}
In TD(0), $\theta^*$ satisfies $||\theta^*|| \le R$ with $R := r_{\rm max}/\omega_{\rm I}$.
\end{lemma}

\begin{lemma}\label{l: decaying random variable 2}
For a sequence of numbers $\{x_{t}\}$ and three constants $a,b,c$ such that $a>1$, we have the following recursive inequality:
\begin{align*}
    x_{t+1} \le \left( 1 - \frac{a}{c+t} \right) x_t + \frac{b^2}{(c+t)^2}, \quad \forall t \ge \tau.
\end{align*}
Then we have the following result:
$$
 x_t \le \frac{\nu}{c + t} \text{, where } \nu = \max \left\{ \frac{b^2}{a-1}, (c+\tau) x_{\tau} \right\}.
$$
\end{lemma}

Given all preliminaries, we can begin our proof. 

\begin{proof}[Proof of Theorem \ref{t: td with markov}]

For simplicity, denote $\Delta_t = \bar{\theta}_t - \theta^*$. Letting $t = T+ t_0$, we take expectation on both sides in (\ref{eq: td 0 update}), 
$$
\Delta_{t+1} = \Delta_t + \alpha_t \mathbb{E} \left[ \bar{g}(\theta_t) \right] + \alpha_t \mathbb{E} \left[ \bar{g}'(\theta_t) - \bar{g}(\theta_t) \right].
$$
By (\ref{eq: g bar}), we know that $\mathbb{E} \left[ \bar{g}(\theta_t) \right] = - \Sigma_{\rm I} \Delta_t$. Notice that, Assumption \ref{a: input assumption} immediately implies $\Vert \Sigma_{\rm I} \Vert \le 1$. Therefore, 
\begin{align*}
    & \Vert (I - \alpha_t \Sigma_{\rm I}) \Delta_t \Vert^2 \\ 
    = & \Vert \Delta_t \Vert^2 - \alpha_t \Delta_t^T (\Sigma_{\rm I}^T + \Sigma_{\rm I}) \Delta_t + \alpha_t^2 \Delta_t^T \Sigma_{\rm I}^T \Sigma_{\rm I} \Delta_t \\
    \le & (1 - 2 \omega_{\rm I} \alpha_t + 4 \alpha_t^2) \Vert \Delta_t \Vert^2 \\
    \le & (1-\omega_{\rm I} \alpha_t) \Vert \Delta_t \Vert^2.
\end{align*}
Here we use Lemma \ref{l: smallest eigenvalue of i.i.d. matrix}, \ref{l: singular value of i.i.d. matrix} and the fact that $4\alpha_t \le \omega_{\rm I}$ (recall we assume $t \ge t_0$). This immediately implies
$$
\Vert (I - \alpha_t \Sigma_{\rm I}) \Delta_t \Vert \le (1-\omega_{\rm I} \alpha_t/2) \Vert \Delta_t \Vert,
$$
where we use the fact $\sqrt{1-x} \le 1-x/2$. Therefore,
$$
\Vert \Delta_{t+1} \Vert \le (1-\omega_{\rm I} \alpha_t/2) \Vert \Delta_t \Vert + \alpha_t \mathbb{E} \left[ \left\Vert \bar{g}'(\theta_t) - \bar{g}(\theta_t) \right\Vert \right].
$$
To address the second term on the right-hand side, by Assumption \ref{a: uniform mixing}, for all $t \ge t_0$, 
\begin{align*}
    \mathbb{E} \left[ \Vert \bar{g}'(\theta_t) - \bar{g}(\theta_t) \Vert \right] 
    \le & \mathbb{E} \bigg[ \left( r_{\rm max} + 2 \Vert \theta_t - \theta^* \Vert + 2 \Vert \theta^* \Vert \right) \cdot \Vert P_t( \cdot|s_0 ) - \mu \Vert_1 \bigg]\\
    \le & \alpha_t \left( r_{\rm max} + 2 u + 2 R \right).
\end{align*}
where $u$ is defined as $u = \max_{i,t} \mathbb{E} \Vert \theta_t(i) - \theta^* \Vert$. Notice that Lemma \ref{l: result of linear td with markov} guarantees $u$ is finite. This immediately indicates 
\begin{align*}
    ||\Delta_{t+1}||
    \le & \left( 1- \omega_{\rm I} \alpha_t/2 \right) ||\Delta_t|| + \alpha_t^2 \left( r_{\rm max} + 2u + 2 R \right).
\end{align*}
We set $x_t = \Vert \Delta_t \Vert, a = \omega_{\rm I}/(\omega(1-\gamma)), b^2 = r_{\rm max} + 2u + 2 R, c=1$ and $\tau = t_0$. By Lemma \ref{l: decaying random variable 2},
$$
||\Delta_t|| \le \frac{\nu}{1 + t},\quad \nu = \max \{ \alpha, \beta \}
$$
where
\begin{align*}
    \alpha = \frac{r_{\rm max} + 2u + 2 R}{\frac{\omega_{\rm I}}{\omega(1-\gamma)}-1}, \quad \beta = & (1+t_0) \Vert \Delta_{t_0} \Vert.
\end{align*}
Since all the above facts holds for every agent $i$, the result directly follows after plugging the above fact as well as Lemma \ref{l: result of linear td with markov} into (\ref{eq: distributed analysis}).

\end{proof}

\subsection{Proof of Lemma \ref{l: smallest eigenvalue of i.i.d. matrix}}

\begin{proof}
Based on \cite{liu2021temporal,tian2022performance}, one can show that
\begin{align*}
x^T \Sigma_I x = & (1-\gamma) \sum_{s \in S} \mu(s)y(s)^2 + \gamma \sum_{s,s' \in S} \mu(s) P(s'|s) (y(s') - y(s))^2 \\
\ge & (1-\gamma) \omega \sum_{s \in S} x(s)^2.    
\end{align*}
Here, $x \in \mathbb{R}^d$ is an arbitrary vector and $y = \Phi x \in \mathbb{R}^{|S|}$, whereas $x(s),y(s)$ is the entry of $x,y$ corresponding to the state $s$. Then, it is obvious that $\omega_{\rm I} \ge (1-\gamma) \omega$.
\end{proof}

\subsection{Proof of Lemma \ref{l: singular value of i.i.d. matrix}}

\begin{proof}
According to the definition of $\mathbb{E}_{\rm I}$ before (\ref{eq: stationary point for td 0}), $\Sigma_{\rm I} x = \mathbb{E}_{\rm I} \left[ (\phi(s)^T x - \gamma \phi(s')^T x) \phi(s) \right]$. Therefore, 
\begin{align*}
    \Vert \Sigma_{\rm I} x \Vert^2 
    \le \mathbb{E}_{\rm I} \left[ \left\Vert (\phi(s)^T x - \gamma \phi(s')^T x) \phi(s) \right\Vert^2 \right] \le 4 \Vert x \Vert^2
\end{align*}
where we use $\Vert \phi(s) \Vert \le 1$ and $\gamma \le 1$.
\end{proof}

\subsection{Proof of Lemma \ref{l: bounded optimal point}}

\begin{proof}
By the Gershgorin circle theorem, $D (\gamma P - I)$ is invertible, and thus so is $\Phi^T D (\gamma P - I) \Phi$. By (\ref{eq: stationary point for td 0}),
$$
\theta^* = \left[ \Phi^T D (I - \gamma P) \Phi \right]^{-1} \Phi^T D R.
$$
By Lemma 5.9 in \cite{olshevsky2023small} and Lemma \ref{l: smallest eigenvalue of i.i.d. matrix}, $\Vert \Sigma_{\rm I} \Vert^{-1} \le \omega_{\rm I}^{-1}$. Furthermore, since $\Vert\Phi^T \sqrt{D}\Vert^2 \le 1$ which is because all features vectors have norm at most one by assumption, and $\Vert \sqrt{D} R \Vert^2 \le r_{\rm max}^2$, we obtain $||\theta^*|| \le r_{\rm max}/\omega_{\rm I}$.

\end{proof}

\subsection{Proof of Lemma \ref{l: decaying random variable 2} }

\begin{proof}
We prove it by induction. First, it is easy to see $ x_\tau \le \frac{\nu}{c+\tau}$. Now suppose $ x_t \le \frac{\nu}{c+t}$, 
\begin{align*}
    x_{t+1} 
    \le & \left( 1 - a \cdot \alpha_t \right) x_t + \frac{b^2}{(c+t)^2} \\
    \le & \left( 1- \frac{a }{c+t} \right) \frac{\nu}{c+t} + \frac{b^2}{(c+t)^2} \\
    = & \frac{c+t-1}{(c+t)^2} \nu + \frac{(1-a)\nu + b^2}{(c+t)^2} \\
    \le & \frac{1}{c+t+1} \nu,
\end{align*}
where the last inequality uses the facts that $x^2 \ge (x-1)(x+1), \forall x$ and $(1-a)\nu+b^2 \le 0$ (This is because of the definition of $\nu$ as defined in Lemma \ref{l: decaying random variable 2}).     
\end{proof}

\section{Proof of Theorem \ref{t: td lambda with markov}}

Our proof here follows the same strategy as for TD(0). For simplicity, we define
$$
\Sigma^{(\lambda)}_{I} := \Phi^T D \Phi - (1-\lambda) \sum_{k=0}^\infty \lambda^k \gamma^{k+1} \Phi^T D P^{k+1} \Phi.
$$
One could use (\ref{eq: t lambda}) and (\ref{eq: stationary point for td lambda}) to show that 
\begin{equation}
    \label{eq: x bar}
    \bar{x}(\theta) = \Phi^T D \left( T_{\pi}^{(\lambda)}(\Phi \theta) - \Phi \theta \right) = \Sigma^{(\lambda)}_{I} (\theta - \theta^*).
\end{equation}
Inspired by \cite{liu2021temporal}, we claim the following lemma whose proof we postpone:

\begin{lemma}
\label{l: smallest eigenvalue of i.i.d. matrix, lambda}
These exists $\omega^{(\lambda)}_{\rm I} > 0$ such that
$$
\inf_{\Vert x \Vert = 1} x^T \Sigma^{(\lambda)}_{\rm I} x \ge \omega^{(\lambda)}_{\rm I} \ge (1-\kappa) \omega.
$$
\end{lemma}

\begin{lemma}
\label{l: singular value of i.i.d. matrix, lambda}
For any $x$, $\Vert \Sigma_{\rm I}^{(\lambda)} x \Vert^2 \le 4 \Vert x \Vert^2$.
\end{lemma}

For the rest of this section, we will also omit $i$ since the analysis holds for all $i \in \{1, 2, \ldots, n\}$. 

As before, we denote $\bar{x}'(\theta_t) = \mathbb{E}_{\rm M} \left[ x_{s,s'}(\theta_t, z_{0:t}) \right]$. This expectation assumes that $s_0 \sim \mu$ while the subsequent states are sampled according to the transition probability of the policy, i.e., $s_k \sim P(\cdot|s_{k-1})$. We call the quantity  $\bar{x}'(\theta_t) - \bar{x}(\theta_t)$ Markov noise. 

Finally, for the stationary point defined in (\ref{eq: stationary point for td lambda}), we have the following lemmas whose proof we postponed:

\begin{lemma} \label{l: bounded optimal point lambda}
In TD($\lambda$), $\theta^*$ satisfies $\Vert \theta^* \Vert \le R^{(\lambda)}$ with $R^{(\lambda)} := \frac{r_{\rm max}}{\omega_{\rm I}^{(\lambda)}(1-\gamma )}$.    
\end{lemma}

Now we are ready to begin the proof.

\begin{proof}[Proof of Theorem \ref{t: td lambda with markov}]

For simplicity, denote $\Delta_t = \bar{\theta}_t - \theta^*$. Assuming $t \ge t_0$, we take expectation on both sides in (\ref{eq: td lambda update}), 
\begin{align*}
    \Delta_{t+1} = \Delta_t + \alpha_t \mathbb{E} \left[ \bar{x}(\theta_t) \right] + \alpha_t \left( \mathbb{E} \left[ \bar{x}'(\theta_t) - \bar{x}(\theta_t) \right] \right). 
\end{align*}
By (\ref{eq: x bar}), we have $\mathbb{E} \left[ \bar{x}(\theta_t) \right] = \Sigma_{\rm I}^{(\lambda)} \Delta_t$. Notice that
\begin{align*}
    & \Vert (I - \alpha_t \Sigma^{(\lambda)}_{\rm I}) \Delta_t \Vert^2 \\ 
    = & \Vert \Delta_t \Vert^2 - \alpha_t \Delta_t^T ({\Sigma^{(\lambda)}_{\rm I}}^T + \Sigma^{(\lambda)}_{\rm I}) \Delta_t + \alpha_t^2 \Delta_t^T {\Sigma^{(\lambda)}_{\rm I}}^T \Sigma^{(\lambda)}_{\rm I} \Delta_t \\
    \le & (1-\omega^{(\lambda)}_{\rm I} \alpha_t) \Vert \Delta_t \Vert^2.
\end{align*}
Here we use both Lemma \ref{l: smallest eigenvalue of i.i.d. matrix, lambda}, \ref{l: singular value of i.i.d. matrix, lambda} and the fact that $4\alpha_t \le \omega^{(\lambda)}_{\rm I}$ (recall we assume $t \ge t^{(\lambda)}_0$). This immediately implies
$$
\Vert (I - \alpha_t \Sigma^{(\lambda)}_{\rm I}) \Delta_t \Vert \le (1-\omega^{(\lambda)}_{\rm I} \alpha_t/2) \Vert \Delta_t \Vert,
$$
where we use the fact $\sqrt{1-x} \le 1-x/2$. Therefore,
$$
\Vert \Delta_{t+1} \Vert \le (1-\omega^{(\lambda)}_{\rm I} \alpha_t/2) \Vert \Delta_t \Vert + \alpha_t \mathbb{E} \left[ \left\Vert \bar{x}'(\theta_t) - \bar{x}(\theta_t) \right\Vert \right].
$$
To deal with the Markov noise, for simplicity, we write 
\begin{equation}
\label{eq: markov noise, lambda}
\begin{aligned}
    \bar{x}(\theta) 
    = & \sum_{k=0}^{+\infty} (\gamma \lambda)^k \sum_{s_{t-k}} \mu(s_{t-k}) \phi(s_{t-k}) l_{t-k}(\theta) \\
    \bar{x}'(\theta) 
    = & \sum_{k=0}^{t} (\gamma \lambda)^k \sum_{s_{t-k}} P_{t-k}(s_{t-k}|s_0) \phi(s_{t-k}) l_{t-k}(\theta),
\end{aligned}
\end{equation}
where $l_{t-k}(\theta) = \mathbb{E}_{s \sim P_k(\cdot|s_{t-k}), s' \sim P(\cdot|s)} \left[ \delta_{s,s'}(\theta) \right]$. Let $l_\theta = | l_{t-k}(\theta) | / (1-\gamma \lambda)$. A simple bound for $l_\theta$ is
$$
l_\theta \le \frac{r_{\rm max} + 2 u^{(\lambda)} + 2R^{(\lambda)}}{1-\gamma \lambda},
$$
where we both use Lemma \ref{l: bounded optimal point lambda} and denote $u^{(\lambda)} = \max_{i,t} \mathbb{E} \Vert \theta_t(i) - \theta^* \Vert$. Notice that Lemma \ref{l: result of td lambda with markov} guarantees that $u^{(\lambda)}$ is finite. With these notations, we have
\begin{align*}
\bar{x}(\theta) - \bar{x}'(\theta) 
= & \sum_{k=0}^{t} (\gamma \lambda)^k \sum_{s_{t-k}} \left[  \mu(s_{t-k}) - P_{t-k}(s_{t-k}|s_0) \right] \phi(s_{t-k}) l_{t-k}(\theta) \\
& \quad + \sum_{k=t+1}^{+\infty} (\gamma \lambda)^k \sum_{s_{t-k}} \mu(s_{t-k}) \phi(s_{t-k}) l_{t-k}(\theta).
\end{align*}
We denote the first term as $I_{0:t}$ and divide it into two terms, $I_{0:\tau_{\rm mix}^{(\lambda)}}$ and $I_{\tau_{\rm mix}^{(\lambda)}+1:t}$. By Assumption \ref{a: uniform mixing}, 
$$
\left\Vert P_{t-k}(\cdot|s_0) - \mu \right\Vert_1 \le \alpha_t, \quad \forall t \le \tau_{\rm mix}^{(\lambda)},
$$
where we use the fact $t\ge 2 \tau^{(\lambda)}_{\rm mix}$. Therefore, 
\begin{align*}
    \mathbb{E} \left[ \Vert I_{0:\tau_{\rm mix}^{(\lambda)}} \Vert \right] 
    \le \mathbb{E} \left[ | l_{t-k}(\theta) | \sum_{k=0}^{\tau_{\rm mix}^{(\lambda)}} (\gamma \lambda)^k \cdot \left\Vert P_{t-k}(\cdot|s_0) - \mu \right\Vert_1  \right] 
    \le \mathbb{E} \left[ l_\theta \right] \alpha_t.
\end{align*}
By the definition of $\tau_{\rm mix}^{(\lambda)}$, $(\gamma \lambda)^{t} \le \alpha_t$. Therefore,
\begin{align*}
    \mathbb{E} \left[ \Vert I_{\tau_{\rm mix}^{(\lambda)}+1:t} \Vert \right] 
    \le \mathbb{E} \left[  | l_{t-k}(\theta) |  \mbox{\hspace{-9pt}}  \sum_{k=\tau_{\rm mix}^{(\lambda)}+1}^{t} \mbox{\hspace{-8pt}} (\gamma \lambda)^k \left\Vert P_{t-k}(\cdot|s_0) - \mu \right\Vert_1 \right] 
    \le 2 \mathbb{E} \left[l_\theta\right] \alpha_t.
\end{align*}
The second term (under expectation) also has upper-bound $\mathbb{E} \left[l_\theta \right] \alpha_t$ since $(\gamma \lambda)^{t} \le \alpha_t$ and the remaining terms are bounded by $\mathbb{E} \left[l_\theta\right]$. So far, we have
\begin{align*}
    \Vert \Delta_{t+1} \Vert 
    = (1-\alpha_t \omega_{\rm I}^{(\lambda)}/2)\Vert \Delta_t \Vert + \alpha_t^2 \frac{4}{1-\gamma \lambda} \cdot \left( r_{\rm max} + 2u^{(\lambda)} + 2R^{(\lambda)} \right),     
\end{align*}
We set $x_t = \Vert \Delta_t \Vert, a = \omega^{(\lambda)}_{\rm I}/(\omega(1-\kappa)), b^2 = 4 \left( r_{\rm max} + 2u^{(\lambda)} + 2 R^{(\lambda)} \right)/(1-\gamma \lambda), c=1$ and $\tau = t_0^{(\lambda)}$. By Lemma \ref{l: decaying random variable 2},
$$
||\Delta_t|| \le \frac{\nu^{(\lambda)}}{1 + t}, \quad \nu^{(\lambda)} = \max \left\{ \alpha^{(\lambda)}, \beta^{(\lambda)} \right\}
$$
where 
\begin{align*}
    \alpha^{(\lambda)} =  \frac{4r_{\rm max} + 8u^{(\lambda)} + 8 R^{(\lambda)}}{(\frac{\omega_{\rm I}^{(\lambda)}}{\omega(1-\kappa)}-1)(1-\gamma \lambda)}, \quad \beta^{(\lambda)} = (1 + t_0^{(\lambda)}) \Vert \Delta_{t_0} \Vert.
\end{align*}
Since all the above facts hold for every agent $i$, the result directly follows after plugging the above fact as well as Lemma \ref{l: result of td lambda with markov} into (\ref{eq: distributed analysis}).
\end{proof}

\subsection{Proof of Lemma \ref{l: smallest eigenvalue of i.i.d. matrix, lambda}}

\begin{proof}
As pointed out in Theorem 2 in \cite{liu2021temporal}, $x^T \Sigma_{\rm I}^{(\lambda)} x$ equals to a convex combination of $D$-norm and Dirichlet semi-norm. Since Dirichlet semi-norm is always no less than zero, we have
$$
x^T \Sigma_{\rm I}^{(\lambda)} x \ge (1-\kappa)x^T \Phi^T D \Phi x, \quad \forall x \in \mathbb{R}^d,
$$
which implies $\omega^{(\lambda)}_{\rm I} \ge (1-\kappa) \omega$.
\end{proof}

\subsection{Proof of Lemma \ref{l: singular value of i.i.d. matrix, lambda}}

\begin{proof}
According to the definition of $\mathbb{E}_{\rm I}$ after (\ref{eq: t lambda}),
$$
\Sigma_{\rm I}^{(\lambda)} x = \mathbb{E}_{\rm I} \left[ \left( \phi(s_0)^T x - (1-\lambda) \sum_{k=0}^\infty \lambda^k \gamma^{k+1} \phi(s_{k+1})^T x \right) \phi(s_0) \right].
$$
Since $\Vert \phi(s) \Vert \le 1$ and $\kappa \le 1$,
\begin{align*}
    \Vert \Sigma_{\rm I}^{(\lambda)} x \Vert
    \le ( 1+(1-\lambda) \sum_{k=0}^\infty \lambda^k \gamma^{k+1} ) \Vert x \Vert = (1+\kappa) \Vert x \Vert \le 2 \Vert x \Vert.
\end{align*}
\end{proof}

\subsection{Proof  of Lemma \ref{l: bounded optimal point lambda}}

\begin{proof}
Solving for $\theta^*$ using (\ref{eq: t lambda}) and (\ref{eq: stationary point for td lambda}), we  arrive at
$$
\theta^* = \left(\Sigma_{\rm I}^{(\lambda)}\right)^{-1} \Phi^T D (1-\lambda) \sum_{k=0}^\infty \lambda^k \sum_{t=0}^k \gamma^t P^t R.
$$
By Lemma 5.9 in \cite{olshevsky2023small} and Lemma \ref{l: smallest eigenvalue of i.i.d. matrix, lambda}, $\Vert \left( \Sigma_{\rm I}^{(\lambda)} \right)^{-1} \Vert \le {\omega^{(\lambda)}_{\rm I}}^{-1}$. Furthermore, since $\Vert\Phi^T \sqrt{D}\Vert^2 \le 1$ which is because all features vectors have norm at most one by assumption, and $\Vert \sqrt{D} (1-\lambda) \sum_{k=0}^\infty \lambda^k \sum_{t=0}^k \gamma^t P^t R \Vert^2 \le r_{\rm max}^2/(1-\gamma)^2$, we obtain $\Vert \theta^* \Vert \le \frac{r_{\rm max}}{\omega_{\rm I}^{(\lambda)}(1-\gamma)}$.

\end{proof}

\section{Proof of Lemma \ref{l: result of linear td with markov} and Lemma \ref{l: result of td lambda with markov}}

In \cite{bhandari2018finite}, Lemma \ref{l: result of linear td with markov} and \ref{l: result of td lambda with markov} are both proved under a projected TD Learning setting. In this section, we will give the proof for both lemmas when the projection is removed. 

In this section, we denote $\theta_t$ as the parameters in step $t$ and omit $i$ since the analysis holds for all $i \in \{1, 2, \ldots, n\}$. With some abuse of notations, we also denote $\Delta_t = \theta_t-\theta^*$ (instead of $\bar{\theta}_t-\theta^*$ in previous sections). 

In TD(0), we also mark that $\sigma := \max_{s_t,s_{t+1}} \Vert g_{s_t, s_{t+1}}(\theta^*) \Vert$. This term is finite since $\Vert \theta^* \Vert$ is both finite. To show this, recall
$$
g_{s_t, s_{t+1}}(\theta^*) = \left( r_t + \gamma \phi(s_{t+1}) \theta^* - \phi(s_t) \theta^* \right) \phi(s_t).
$$
Therefore, 
$$
\Vert g_{s_t, s_{t+1}}(\theta^*) \Vert \le r_{\rm max} +2 \Vert \theta^* \Vert \le r_{\rm max} + 2R
$$
where $R$ is given in Lemma \ref{l: bounded optimal point}. 

In TD($\lambda$), we will mark $\sigma^{(\lambda)} := \max_{s_t,s_{t+1}} \Vert x_{s_t, s_{t+1}}(\theta^*, z_{0:t}) \Vert$. We next show that this term is also finite. Recall 
$$
x_{s_t, s_{t+1}}(\theta^*, z_{0:t}) = \left( r_t + \gamma \phi(s_{t+1}) \theta^* - \phi(s_t) \theta^* \right) \sum_{k=0}^t (\gamma \lambda)^k \phi(s_{t-k}).
$$
Therefore, 
$$
\Vert g_{s_t, s_{t+1}}(\theta^*) \Vert \le \frac{r_{\rm max} +2 \Vert \theta^* \Vert}{1-\gamma \lambda} \le \frac{r_{\rm max} + 2R^{(\lambda)}}{1-\gamma \lambda}
$$
where $R^{(\lambda)}$ is given in Lemma \ref{l: bounded optimal point lambda}. 

\subsection{Proof of Lemma \ref{l: result of linear td with markov}}

\begin{proof}

Suppose $t \ge t_{\rm th}$. TD(0) update implies
$$
\mathbb{E} \left[ \Vert \Delta_{t+1} \Vert^2 \right] = \mathbb{E} \left[ \Vert \Delta_t \Vert^2 \right] + 2 \alpha_t \mathbb{E} \left[ \Delta_{t}^T \bar{g}(\theta_t) \right] + \alpha_t^2 \mathbb{E} \left[ \Vert g_{s_t, s_{t+1}}(\theta_t) \Vert^2 \right] + 2 \alpha_t \mathbb{E} \left[ \Delta_{t}^T \left(  g_{s_t, s_{t+1}}(\theta_t) - \bar{g}(\theta_t) \right) \right]. 
$$
For the second term on the right hand side, by (\ref{eq: x bar}) and Lemma \ref{l: singular value of i.i.d. matrix, lambda}, we conclude that
$$
\mathbb{E} \left[ \Delta_t^T \bar{g}(\theta_t) \right] \le - (1-\gamma) \omega \mathbb{E} \left[ \Vert \Delta_t \Vert^2 \right].
$$
For the third term on the right hand side, notice that
$$
\Vert g_{s_t, s_{t+1}}(\theta_t) \Vert^2 \le 2 \Vert g_{s_t, s_{t+1}}(\theta_t) - g_{s_t, s_{t+1}}(\theta^*) \Vert^2 + 2 \Vert g_{s_t, s_{t+1}}(\theta^*) \Vert^2.
$$
We bound the first term on the right hand side using the following: 
$$
\Vert g_{s_t, s_{t+1}}(\theta_t) - g_{s_t, s_{t+1}}(\theta^*) \Vert^2 = \left\Vert \left( \left(\gamma \phi(s_{t+1}) - \phi(s_t) \right)^T \Delta_t  \right) \phi(s_t) \right\Vert \le 4 \Vert \Delta_t \Vert^2.
$$
Therefore, 
$$
\mathbb{E} \left[ \Vert g_{s_t, s_{t+1}}(\theta_t) \Vert^2 \right] \le 8 \mathbb{E} \left[ \Vert \Delta_t \Vert^2 \right] + 2\sigma^2
$$
For the fourth term on the right hand side, from the above analysis, we obtain a quick result
$$
\Vert g_{s_t, s_{t+1}}(\theta_t) \Vert \le 4 \Vert \Delta_t \Vert + 2 \sigma.
$$
Recall the definition of $\bar{g}$ and $\bar{g}'$ in (\ref{eq: markov noise}), 
$$
\mathbb{E} \left[ g_{s_t, s_{t+1}}(\theta_t) - \bar{g}(\theta_t) | \Delta_t \right] = \bar{g}'(\theta_t) - \bar{g}(\theta_t) = \sum_{s_t, s_{t+1}} \left( P_t(s_t|s_0) - \mu(s_t) \right) P(s_{t+1}|s_t) g_{s_t, s_{t+1}}(\theta_t).
$$
Therefore, 
\begin{align*}
    \mathbb{E} \left[ \Delta_{t}^T \left(  g_{s_t, s_{t+1}}(\theta_t) - \bar{g}(\theta_t) \right) \right] 
    = & \mathbb{E} \left[ \Delta_{t}^T \cdot \mathbb{E} \left[ g_{s_t, s_{t+1}}(\theta_t) - \bar{g}(\theta_t) | \Delta_t \right] \right] \\
    = & \mathbb{E} \left[ \Vert \Delta_{t} \Vert \cdot \mathbb{E} \left[ \Vert g_{s_t, s_{t+1}}(\theta_t) - \bar{g}(\theta_t) \Vert | \Delta_t \right] \right] \\
    \le & \mathbb{E} \left[ \Vert \Delta_t \Vert \cdot \Vert P_t(\cdot|s_0) - \mu \Vert_1 \cdot \left( 4 \Vert \Delta_t \Vert + 2 \sigma \right)  \right] \\
    \le & 4 \alpha_t \mathbb{E} \left[ \Vert \Delta_t \Vert^2 \right] + 2 \sigma \alpha_t \mathbb{E} \left[ \Vert \Delta_t \Vert \right] \\
    \le & 4 \alpha_t \mathbb{E} \left[ \Vert \Delta_t \Vert^2 \right] + \alpha_t \sigma^2 + \alpha_t \mathbb{E} \left[ \Vert \Delta_t \Vert^2 \right] \\
    = & 5 \alpha_t \mathbb{E} \left[ \Vert \Delta_t \Vert^2 \right] + \alpha_t \sigma^2
\end{align*}
where the fourth line uses Assumption \ref{a: uniform mixing} since we are assuming $t \ge \tau_{\rm mix}$. Now combine the results we have so far, 
$$
\mathbb{E} \left[ \Vert \Delta_{t+1} \Vert^2 \right] \le \left( 1 - 2 (1-\gamma) \omega \alpha_t + 18 \alpha_t^2 \right) \mathbb{E} \left[ \Vert \Delta_t \Vert^2 \right] + 4 \alpha_t^2 \sigma^2. 
$$
If $\alpha_t \le (1-\gamma)\omega/9$ (which is true since we are assuming $t \ge \frac{18}{(1-\gamma)^2 \omega^2}-1$), we know
$$
\mathbb{E} \left[ \Vert \Delta_{t+1} \Vert^2 \right] = \left( 1-\alpha_t(1-\gamma) \omega \right) \mathbb{E}\left[ \Vert \Delta_t \Vert^2 \right] + 4 \alpha_t^2 \sigma^2. 
$$
Therefore, we are now able to use Lemma \ref{l: decaying random variable 2}. We set $x_t = \mathbb{E} \left[ \Vert \Delta_t \Vert^2 \right], a = 2, b^2 = 16 \sigma^2 / (\omega^2 (1-\gamma)^2), c = 1$ and $\tau = \tau_{\rm mix}$. By Lemma \ref{l: decaying random variable 2}, 
$$
\mathbb{E} \left[ \Vert\Delta_t \Vert^2 \right] \le \frac{\nu}{t+1}, \quad \nu = \max \{ \frac{16 \sigma^2}{\omega^2 (1-\gamma^2)}, \left( \tau_{\rm mix} + 1 \right) \mathbb{E} \left[ \Vert \Delta_{\tau_{\rm mix}} \Vert^2 \right] \}.
$$

\end{proof}

\subsection{Proof of Lemma \ref{l: result of linear td with markov}}

\begin{proof}
Suppose $t \ge t_{\rm th}^{(\lambda)}$. TD($\lambda$) update implies
$$
\mathbb{E} \left[ \Vert \Delta_{t+1} \Vert^2 \right] = \mathbb{E} \left[ \Vert \Delta_t \Vert^2 \right] + 2 \alpha_t \mathbb{E} \left[ \Delta_{t}^T \bar{x}(\theta_t) \right] + \alpha_t^2 \mathbb{E} \left[ \Vert x_{s_t, s_{t+1}}(\theta_t, z_{0:t}) \Vert^2 \right] + 2 \alpha_t \mathbb{E} \left[ \Delta_{t}^T \left(  x_{s_t, s_{t+1}}(\theta_t, z_{0:t}) - \bar{x}(\theta_t) \right) \right]. 
$$
For the second term on the right hand side, by (\ref{eq: g bar}) and Lemma \ref{l: singular value of i.i.d. matrix}, we conclude that
$$
\mathbb{E} \left[ \Delta_t^T \bar{x}(\theta_t) \right] \le - (1-\kappa) \omega \mathbb{E} \left[ \Vert \Delta_t \Vert^2 \right].
$$
For the third term on the right hand side, notice that
$$
\Vert x_{s_t, s_{t+1}}(\theta_t, z_{0:t}) \Vert^2 \le 2 \Vert x_{s_t, s_{t+1}}(\theta_t, z_{0:t}) - x_{s_t, s_{t+1}}(\theta^*, z_{0:t}) \Vert^2 + 2 \Vert x_{s_t, s_{t+1}}(\theta^*, z_{0:t}) \Vert^2.
$$
We bound the first term on the right hand side using the following: 
$$
\Vert x_{s_t, s_{t+1}}(\theta_t, z_{0:t}) - x_{s_t, s_{t+1}}(\theta^*, z_{0:t}) \Vert^2 = \left\Vert \left( \left(\gamma \phi(s_{t+1}) - \phi(s_t) \right)^T \Delta_t  \right) \sum_{k=0}^{t} (\gamma \lambda)^k \phi(s_{t-k}) \right\Vert \le \frac{4}{1-\gamma \lambda} \Vert \Delta_t \Vert^2.
$$
Therefore, 
$$
\mathbb{E} \left[ \Vert x_{s_t, s_{t+1}}(\theta_t, z_{0:t}) \Vert^2 \right] \le \frac{8}{1-\gamma \lambda} \mathbb{E} \left[ \Vert \Delta_t \Vert^2 \right] + 2{\sigma^{(\lambda)}}^2
$$
For the fourth term on the right hand side, from the above analysis (and the fact $1-\gamma \lambda \le 1$), we obtain a quick result
$$
\Vert x_{s_t, s_{t+1}}(\theta_t, z_{0:t}) \Vert \le \frac{4}{1-\gamma \lambda} \Vert \Delta_t \Vert + 2 {\sigma^{(\lambda)}}.
$$
Recall the definition of $\bar{x}$ in (\ref{eq: markov noise, lambda}), 
\begin{align*}
    & \mathbb{E} \left[ x_{s_t, s_{t+1}}(\theta_t, z_{0:t}) - \bar{x}(\theta_t) | \Delta_t \right] \\
    = & \sum_{s_t, s_{t+1}} P_t(s_t|s_0) P(s_{t+1}|s_t) x_{s_t, s_{t+1}}(\theta_t, z_{0:t}) - \sum_{s_t, s_{t+1}} \mu(s_t) P(s_{t+1}|s_t) x_{s_t, s_{t+1}}(\theta_t, z_{-\infty:t}) \\
    = & \underbrace{\sum_{s_t, s_{t+1}} \left( P_t(s_t|s_0) - \mu(s_t) \right) P(s_{t+1}|s_t) x_{s_t, s_{t+1}}(\theta_t, z_{0:t})}_{I_1} - \underbrace{ (\gamma \lambda)^{t} \sum_{s_t, s_{t+1}} \mu(s_t) P(s_{t+1}|s_t) x_{s_t, s_{t+1}}(\theta_t, z_{-\infty:-1}) }_{I_2}.
\end{align*}
where we denote $z_{-\infty:-1} = \sum_{k=1}^{\infty} (\gamma \lambda)^{k} \phi(s_{-k})$. Notice that one can easily show $x_{s_t, s_{t+1}}(\theta_t, z_{-\infty:-1})$ has the same upper-bound as $x_{s_t, s_{t+1}}(\theta_t, z_{0:t})$. 

We find that $I_1$ satisfies
$$
\Vert I_1 \Vert \le \Vert P_t(\cdot | s_0) - \mu \Vert_1 \cdot \Vert x_{s_t, s_{t+1}}(\theta_t, z_{0:t}) \Vert \le \frac{4}{1-\gamma \lambda} \alpha_t \Vert \Delta_t \Vert + 2 \alpha_t {\sigma^{(\lambda)}}
$$
and $I_2$ satisfies
$$
\Vert I_2 \Vert \le (\gamma \lambda)^t \cdot \Vert x_{s_t, s_{t+1}}(\theta_t, z_{-\infty:-1}) \Vert \le \frac{4}{1-\gamma \lambda} \alpha_t \Vert \Delta_t \Vert + 2 \alpha_t {\sigma^{(\lambda)}}.
$$
Here, in both inequalities we use $t \ge \tau_{\rm mix}^{(\lambda)}$ and, as a result, $\Vert P_t(\cdot | s_0) - \mu \Vert_1 \le \alpha_t$ and $(\gamma \lambda)^t \le \alpha_t$. Therefore, 
\begin{align*}
    \mathbb{E} \left[ \Delta_{t}^T \left(  x_{s_t, s_{t+1}}(\theta_t, z_{0:t}) - \bar{x}(\theta_t) \right) \right] 
    = & \mathbb{E} \left[ \Delta_{t}^T \cdot \mathbb{E} \left[ x_{s_t, s_{t+1}}(\theta_t, z_{0:t}) - \bar{x}(\theta_t) | \Delta_t \right] \right] \\
    \le & \mathbb{E} \left[ \Vert \Delta_{t} \Vert \cdot \mathbb{E} \left[ \Vert x_{s_t, s_{t+1}}(\theta_t, z_{0:t}) - \bar{x}(\theta_t) \Vert | \Delta_t \right] \right] \\
    \le & \mathbb{E} \left[ \Vert \Delta_t \Vert \cdot \left( \Vert I_1 \Vert + \Vert I_2 \Vert \right) \right] \\
    \le & \frac{8}{1-\gamma \lambda} \alpha_t \mathbb{E} \left[ \Vert \Delta_t \Vert^2 \right] + \frac{4}{1-\gamma \lambda} {\sigma^{(\lambda)}} \alpha_t \mathbb{E} \left[ \Vert \Delta_t \Vert \right] \\
    \le & \frac{8}{1-\gamma \lambda} \alpha_t \mathbb{E} \left[ \Vert \Delta_t \Vert^2 \right] + \frac{2}{1-\gamma \lambda} \alpha_t {\sigma^{(\lambda)}}^2 + \frac{2}{1-\gamma \lambda} \alpha_t \mathbb{E} \left[ \Vert \Delta_t \Vert^2 \right] \\
    = & \frac{10}{1-\gamma \lambda} \alpha_t \mathbb{E} \left[ \Vert \Delta_t \Vert^2 \right] + \frac{2}{1-\gamma \lambda} \alpha_t {\sigma^{(\lambda)}}^2
\end{align*}
Now combine the results we have so far (and the fact $1-\gamma \lambda \le 1$), 
$$
\mathbb{E} \left[ \Vert \Delta_{t+1} \Vert^2 \right] \le \left( 1 - 2 (1-\kappa) \omega \alpha_t + \frac{28}{1-\gamma \lambda} \alpha_t^2 \right) \mathbb{E} \left[ \Vert \Delta_t \Vert^2 \right] + \frac{6}{1-\gamma \lambda} \alpha_t^2 {\sigma^{(\lambda)}}^2. 
$$
If $\alpha_t \le (1-\kappa)(1-\gamma \lambda)\omega/14$ (which is true since we are assuming $t \ge \frac{28}{(1-\kappa)^2\omega^2(1-\gamma \lambda)}-1$), we know
$$
\mathbb{E} \left[ \Vert \Delta_{t+1} \Vert^2 \right] = \left( 1-\alpha_t(1-\kappa) \omega \right) \mathbb{E}\left[ \Vert \Delta_t \Vert^2 \right] + \frac{6}{1-\gamma \lambda} \alpha_t^2 {\sigma^{(\lambda)}}^2. 
$$
Therefore, we are now able to use Lemma \ref{l: decaying random variable 2}. We set $x_t = \mathbb{E} \left[ \Vert \Delta_t \Vert^2 \right], a = 2, b^2 = 36 {\sigma^{(\lambda)}}^2 / (\omega^2 (1-\kappa)^2(1-\gamma \lambda)^2), c = 1$ and $\tau = \tau_{\rm mix}$. By Lemma \ref{l: decaying random variable 2}, 
$$
\mathbb{E} \left[ \Vert\Delta_t \Vert^2 \right] \le \frac{\nu}{t+1}, \quad \nu = \max \{ \frac{16 {\sigma^{(\lambda)}}^2}{\omega^2 (1-\kappa^2) (1-\gamma \lambda)^2}, \left( \tau_{\rm mix} + 1 \right) \mathbb{E} \left[ \Vert \Delta_{\tau_{\rm mix}} \Vert^2 \right] \}.
$$
\end{proof}

\section{Numerical Results}

In this section, we provide some numerical results. The goal is to shown that Multi-agent TD Learning does improve the convergence. 

We produce simulations on a five-state Random Walk task (Example 6.2 in \cite{sutton2018reinforcement}). In this task, there are five states in a row, named A, B, C, D, and E, from left to right. All episodes start in the center state, C, then proceed either left or right by one state on each step, with equal probability. Episodes terminate either on the extreme left or the extreme right. When an episode terminates on the right, a reward of +1 occurs; all other rewards are zero. It is easy to show that the true value of all the states, A through E, are 1/6, 2/6, 3/6, 4/6, and 5/6. 

We test both Multi-agent TD(0) and Multi-agent TD($\lambda$) on this task. For each algorithm, we run a total of 100 episodes. The stepsize is chosen to be $1\times 10^{-3}$. The number of agent, $N$, is chosen to be 1, 2, 4, and 8 in each task. 

\begin{figure}
     \centering
     \begin{subfigure}[b]{0.45\textwidth}
         \centering
         \includegraphics[width=\textwidth]{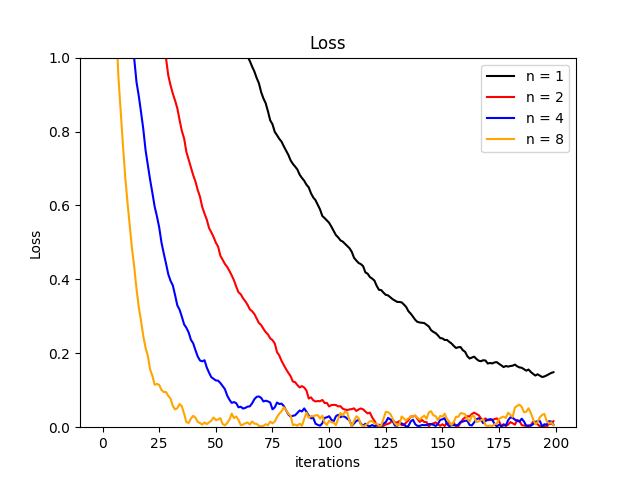}
         \caption{Learning curve}
     \end{subfigure}
     \begin{subfigure}[b]{0.45\textwidth}
         \centering
         \includegraphics[width=\textwidth]{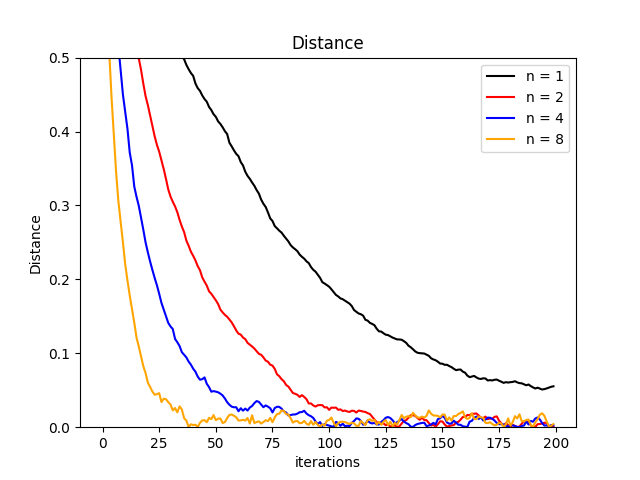}
         \caption{Distances}
     \end{subfigure}
    \caption{Numerical results for TD(0)}
    \label{fig: td 0}
\end{figure}

Figure \ref{fig: td 0} shows the result for Multi-agent TD(0) and Figure \ref{fig: td lambda} shows the result for Multi-agent TD($\lambda$). The figures on the left show the learning curves for different number of agents, which is denoted by n. In these figures, the $x$-axis plots the episodes while the $y$-axis plots the root mean-squared (RMS) error between the value function learned and the true value function, averaged over the five states, then averaged over $N$ agents. This figure shows Multi-agent TD Learning converges faster than single-agent TD Learning. 

\begin{figure}
     \centering
     \begin{subfigure}[b]{0.45\textwidth}
         \centering
         \includegraphics[width=\textwidth]{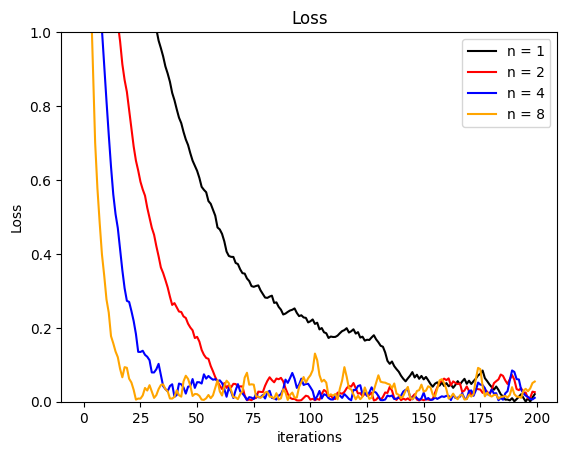}
         \caption{Learning curve}
     \end{subfigure}
     \begin{subfigure}[b]{0.45\textwidth}
         \centering
         \includegraphics[width=\textwidth]{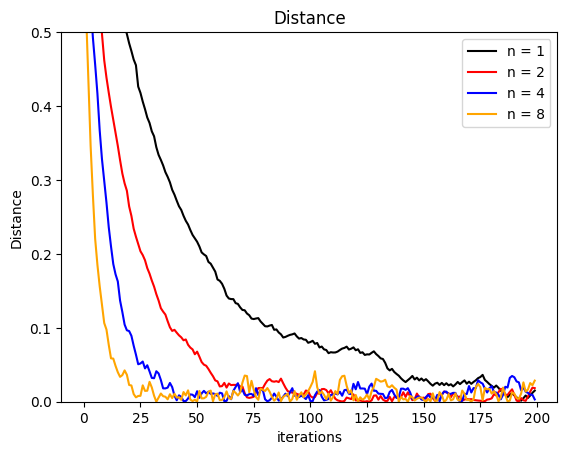}
         \caption{Distances}
     \end{subfigure}
    \caption{Numerical results for TD($\lambda$)}
    \label{fig: td lambda}
\end{figure}

In the right figures, the $x$-axis plots the episodes while the $y$-axis plots the Euclidean distance between the averaged parameters $\bar{\theta}$ and the stationary point $\theta^*$. These figures show that the distance for Multi-agent TD Learning is basically $N$ times less than the distance for single-agent TD Learning, which basically matches what the theorem we propose in our paper.

\section{Conclusions} 

We have shown that one-shot averaging suffices to give a linear speedup for distributed TD($\lambda$) under Markov sampling. This is an improvement over previous works, which had alternatively either $O(T)$ communication rounds per $T$ steps or $O(N)$ averaging rounds per $T$ steps to achieve the same. An open question is whether a similar result can be proven for tabular Q-learning.

\bibliography{iclr2023_conference}
\bibliographystyle{iclr2023_conference}

\newpage

\end{document}

%% file: math_commands.tex
\usepackage{amsmath,amsfonts,bm}

\def\eqref#1{equation~\ref{#1}}

\def\1{\bm{1}}

\DeclareMathAlphabet{\mathsfit}{\encodingdefault}{\sfdefault}{m}{sl}
\SetMathAlphabet{\mathsfit}{bold}{\encodingdefault}{\sfdefault}{bx}{n}

